\documentclass[review]{elsarticle}

\usepackage{lineno,hyperref}
\modulolinenumbers[5]

\usepackage{algorithm}
\usepackage{algpseudocode}
\usepackage{graphicx}
\usepackage{caption}
\usepackage{subcaption}
\usepackage{amsmath}
\usepackage{natbib}
\usepackage{url}
\usepackage{amsthm}
\usepackage{subcaption}

\newtheorem{lem}{Lemma}

\journal{Journal of Parallel and Distributed Computing}

\begin{document}

\begin{frontmatter}

\title{A Novel Co-design Peta-scale Heterogeneous Cluster for Deep Learning Training}


\author[firstaddress]{Xin Chen\corref{cor1}}
\ead{chen1.xin@midea.com}

\author[firstaddress]{Hua Zhou\corref{cor1}}
\ead{hua.zhou@midea.com}
\author[firstaddress]{Yuxiang Gao}
\ead{yuxiang1.gao@midea.com}
\author[firstaddress]{Yu Zhu}
\ead{zhu.yu@midea.com}

\cortext[cor1]{These authors contributed equally to this work. Xin Chen (chen1.xin@midea.com) is the corresponding author. }

\address[firstaddress]{Emerging Technology Center, Midea Corporate Research Center, San Jose, CA, USA}

\begin{abstract}
Large scale deep Convolution Neural Networks~(CNNs) increasingly demands the computing power. It is key for researchers to own a great powerful computing platform to leverage deep learning~(DL) advancing.On the other hand, as the commonly-used accelerator, the commodity GPUs cards of new generations are more and more expensive. Consequently, it is of importance to design an affordable distributed heterogeneous system that provides powerful computational capacity and develop a well-suited software that efficiently utilizes its computational capacity. In this paper, we present our co-design distributed system including a peta-scale GPU cluster, called \textit{``Manoa"}. Based on properties and topology of Manoa, we first propose job server framework and implement it, named \textit{``MiMatrix"}. The central node of MiMatrix, referred to as the job server, undertakes all of controlling, scheduling and monitoring, and I/O tasks without weight data transfer for AllReduce processing in each iteration. Therefore, MiMatrix intrinsically solves the bandwidth bottleneck of central node in parameter server framework that is widely used in distributed DL tasks. Meanwhile, we also propose a new AllReduce algorithm, \textbf{G}PU\textbf{D}irect \textbf{R}DMA-\textbf{A}ware \textbf{A}llReduce~(GDRAA), in which both computation and handshake message are $O(1)$ and the number of synchronization is two in each iteration that is a theoretical minimum number. Owe to the dedicated co-design distributed system, MiMatrix efficiently makes use of the Manoa's computational capacity and bandwidth. We benchmark Manoa Resnet50 and Resenet101 on Imagenet-1K dataset. Some of results have demonstrated state-of-the-art.  

\end{abstract}


\end{frontmatter}

\section{Introduction}
\label{intro}
In recent years, deep Convolution Neural Networks~(CNNs) have prevailed in both academia and industry, and the CNNs models not only have outperformed most of traditional machine learning and pattern recognition techniques~\cite{lecun2015deep,schmidhuber2015deep,liu2017survey} such as computer vision~\cite{krizhevsky2012imagenet,he2016deep,bengio2009learning} and speech recognition~\cite{hinton2012deep,amodei2016deep} but also has ranked number one in the game of Go~\cite{silver2016mastering,silver2017mastering}. With increasing of training dataset size~\cite{russakovsky2015imagenet,cocodatset} and deep learning~(DL) models complexity~\cite{he2016deep,simonyan2014very,huang2017densely}, DL training has become one of the most computationally-demanding high performance computing~(HPC) applications. Consequently, works relevant to build powerful computing platforms in both hardware and software have become hot research fields. In this paper, we address our ad hoc solution, a co-design peta-scale distributed heterogeneous system including an affordable GPU cluster that is less than 1 million dollar, named \textbf{``Manoa"} and a novel job sever software framework, named \textbf{``MiMatrix"}, which is designed based on proprieties of Manoa. It also has other two advantages: 1) high density-all nodes and switches equipped in two 48U racks, and 2) high price and coverage speed ratio.

Compared to other gradient descent optimization algorithms~\cite{ruder2016overview}\footnote {\label{footnote1}In this paper, Asynchronous Stochastic Gradient Descent~(ASGD) and SSGD are categorized into two different optimization algorithms.}, Synchronous Stochastic Gradient Descent~(SSGD) has two major advantages: 1) Obtain the highest accuracy in most cases~\cite{zhang2018theory}; and 2) Guarantee to learn a function in polynomial time~\cite{daniely2017sgd,revistingSGD}. Particularly, most of less-resource methods such as distilling~\cite{distilling} and pruning~\cite{han2015learning,Laine2017} heavily rely on the accuracy of the pre-train models. Considered it, this paper targets at designing and developing a feasible solution to DL training based on SSGD. 

Prior to introduction to our work, we briefly revisit three major challenges of designing both a distributed system and parallel software framework listed as following:\footnote{\label{footnote2}In the rest of paper, if we don't point out the type of distributed DL training approach, the DL training means to SSGD approach.} 
\begin{enumerate}
	\item It is difficult to scale out DL training. Most of large scale CNNs training approaches use mini-batch training\footnote {\label{footnote10}In this paper, we take data parallel training processing.}. From computation's perspective, it should have a higher parallel efficiency with a bigger batch size and scale out to more nodes. On the contrary, too bigger size leads to slow the converge speed of the learning, and even low the accuracy of the trained models~\cite{szegedy2017inception, masters2018revisiting, cho2017powerai}. 
	\item The distributed DL training meets problems of computation bound~(convolution layers), memory bound~(fully-connected layer) and bandwidth bound~(data aggregation among workers)~\cite{zhang2016parallel}, alternately. It causes a dilemma of configuring the hardware of the system and developing and optimizing algorithms of CNNs training;
	\item The DL training process has frequent I/O operations and synchronizations, loading mini-batch and synchronizing model weight in each iteration, which greatly decreases the utilization of the whole system. In turn, it causes the poor performance of the system.
\end{enumerate}

In order to solve the three challenges mentioned above,  inspired by Harvard architecture, we design and build Manoa with a dedicated novel topology, which provides a powerful computational capacity, over 1.2 PetaFLOPS~(PFLOPS) single precision. Meanwhile, given the properties of Manoa, we develop and implement MiMatrix, which is a job sever framework and  maximizes the Manoa's capacity to expedite DL training.  In MiMatrix, we also propose a new parallel SSGD algorithm, referred to as \textbf{G}PU\textbf{D}irect \textbf{R}DMA-\textbf{A}ware \textbf{A}llReduce~(GDRAA). Its computation and handshake message are $O(1)$.  

As shown in Fig.\ref{fig:arch}, analogy to Harvard architecture, Manoa has a head node, referred to as \textbf{Job Server}, which acts on the role of control unit~(CU) for the whole system. The node has powerful CPUs dealing with task-intensive works such as job schedule, system monitoring, real-time visualization, and I/O tasks. Manoa also has sixteen computing nodes, referred to as \textbf{computing server}, that are designed as the role of arithmetic logic unit~(ALU) in Harvard architecture and takes on compute-intensive portions. Besides running forward and backward parts during training, the computer servers act as both client and server roles during AllReduce. Besides them, Manoa also has a 240TB storage node with RAID10. All nodes are connected by both InfiniBand~(IB) and Ethernet. 

We first propose and develop a \textbf{Job Server} parallel software framework, named MiMatrix. This software is co-designed with dedicated topology of Manoa for fully utilizing its computational capacity. Also we propose a novel GPU direct RDMA-aware AllReduce algorithm, SGRAA, short from \textbf{G}PU\textbf{D}irect \textbf{R}DMA-\textbf{A}ware \textbf{A}llReduce, and directly implement the algorithm with IB Verbs library~\cite{rdmaprogramming1}, referred to as \textit{ibverbs}.
\begin{figure}[!t]
\centering
\includegraphics[width=0.9 \textwidth]{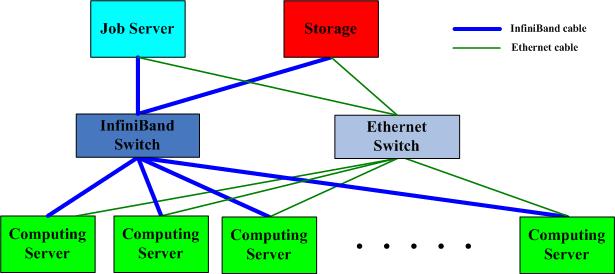}
\caption{Basic scheme of architecture of Manoa. Manoa consists of one job server, sixteen computing servers, one storage (120TB). One 10G Ethernet switch and one 56G InfiniBand switch. The data of models are directly transferred via InfiniBand with GPU Direct RDMA. The message passes through Ethernet.}
\label{fig:arch}
\end{figure}

Along with quick DL development, some distributed systems~\cite{coates2013deep,chilimbi2014project,kurth2017deep,bhattacharjee2017distributed} and parallel software frameworks~\cite{dean2012large,chen2015mxnet,caffe2,Paddlepaddle,MPItheano,tensorflow,CNTK,xing2015petuum} have been released successfully in recent years. However, all of them either design a cluster using some existing general DL SDKs or developed a general distributed SDK running on various distributed systems. It is obvious that the whole system cannot be fully optimized. To the best of our knowledge, Manoa and Mimatrix are firstly co-designed and co-developed in both hardware and software for accelerating distributed DL training.

In this paper, we make two major contributions as followings:
\begin{enumerate}
	\item To meet the increasing computational demands of DL training, we first co-design and co-develop a peta-scale heterogeneous cluster, Manoa, and a job server DL parallel framework, MiMatrix. Our system is a high coverage speed price ratio~(CS/P) and high-density cluster. The price of Manoa is 900K dollar\footnote {\label{footnote3} Price in March, 2017.} and less than 45\% of Nvidia DGX 1 solution\footnote {\label{footnote4} \url{https://www.nvidia.com/en-us/data-center/dgx-1/}}. All equipments are installed in two 48U racks. To the best of our knowledge, it is the highest density of a distributed GPU cluster for DL training.
	\item We propose a novel GDRAA algorithm, in which both computation and handshake message are $O(1)$. We implement our proposed algorithm with \textit{ibverbs}, natively, which fully utilizes the bidirectional bandwidth of all computer servers and reduces latency of data copy.
\end{enumerate}

The rest of the paper is organized as follows. Section~\ref{sec:system} conceptually describes how we design Manoa and MiMatrix and briefly introduces our software implementation. Section~\ref{sec:algorithm} details our proposed GDRAA AllReduce algorithm and proves its properties. In section~\ref{sec:results}, we present our experimental results and some analysis. This paper closes with a conclusion of our work and some future directions in section~\ref{sec:conclusion}.

\section{System Design and Software Implementation}
\label{sec:system}
In the beginning of this section, we address conceptual description of idea and consideration of the co-design distributed heterogeneous system. Following it, we describe our implementation of MiMatrix. 

\subsection{Description of system design and consideration}
A parallelism of a distributed heterogeneous cluster is generally categorized into three levels: 1): 1st level-\textit{worker level}, spanning across workers in a system; 2): 2nd level-\textit{processor level}, spanning across processors in a worker; and 3): 3rd level-\textit{core level}, spanning across cores in a processor. Since collective communications library (NCCL)~\cite{nccl} and CUDNN library~\cite{chetlur2014cudnn} have handled 2nd level and 3rd level tasks, respectively, we focus on 1st level parallelism design and development.

\textbf{Objective of Co-design System and Considerations} Our goal is to design and develop an affordable distributed system, which provides enough computational power to finish most of deep CNNs models training on a large scale training dataset in one day. We take Resnet101 and Resnet50~\cite{he2016deep} as the model benchmark and ImageNet-1K~\cite{russakovsky2015imagenet} as the dataset benchmark. Since GPUs are well-suited for the types of computation of deep CNN, GPUs are widely taken for DL training~\cite{chen2014big}. Additional advantages of GPUs over CPUs include more computational units and a higher bandwidth. Consequently, We choose GPUs as the accelerator of Manoa. 

In this paper, we designed Manoa in November 2016. At that time, the Nvida P100 GPU card is the best GPU card for DL training\footnote{\label{footnote5} At that time, the price of Nvidia Tesla P100 is about \$5,500 per card.} . Consequently, we took the Nvida P100 card with 16G memory as our GPU accelerator. Also, at that time, the last generation of CPU is Broadwell CPU that has forty slots. We selected Mellanox FDR56 InfiniBand adapter cards and switch. Given the number of InfiniBand switch and price, we chose 128 P100 GPU cards and 16 computing nodes, Each node has eight GPUs cards and two InfiniBand adapter cards.  

\textbf{Scaling Efficiency Measurement} Distributed DL system aims at speeding up to converge. As a result, any high computational scaling efficient without being measured to converge to a desirable accuracy is mindless for distributed system for machine learning. In this paper, we measure the scaling using a ratio between an acceptable accuracy and training time, as did~\cite{cho2017powerai}.

In order to measure the affordable machine, we first define a novel criterion, \textbf{p}rice and \textbf{c}onvergence \textbf{r}atio (PCR), defined as:
\begin{equation}
\label{eq:pc}
PC = \frac{1}{time} \times \frac{1}{price} 
\end{equation}
where time is how long the Resnet101 obtains 70\% or more accuracy on validation data of Imagenet 1K, and price is the how much the whole system hardware price, the unit is Kilo Dollar (K\$). PCR  reveal how many dollars one minute convergence demands. The bigger is better. 

\textbf{Manoa Components} The Manoa consists of one job server, sixteen computing servers, and one storage, all of which are connect by Ethernet and IB FDR56, as shown in Fig.\ref{fig:arch}. 

The storage has 120 TB storage of 240TB hard disk drives with RAID10. 

The job server has two Intel Xeon high-end Broadwell CPUs with 512 GB memory, and has over 2TB SSD storage.

As shown in Fig.~\ref{fig:archnode}, each of computing servers has two CPUs. The motherboard is non-uniform memory access~(NUMA). The computer server has eight Nivida Tesla GPUs, and four of them connect to one socket through two PCIE switches. Also there are one Infiniband host channel adapter~(HCA) located on each socket. P100 GPUs card has 9.3 TeraFLOPS single-precision performance and 18.7 TeraFLOPS half-precision performance, the computational capacity of only GPUs of the system has over 1.19 PetaFLOPS for single precision and 2.39 PetaFLOPS for half precision. Plus the CPUs of computer servers and job server, Manoa has over 1.2 PetaFLOPS for the single precision float. 

All equipments are installed in two 48U racks. To the best of our knowledge, it is the highest density of a distributed GPU cluster fro DL training~\cite{kurth2017deep,cho2017powerai}.
\begin{figure}[!t]
	\centering
	\includegraphics[width=0.9 \textwidth]{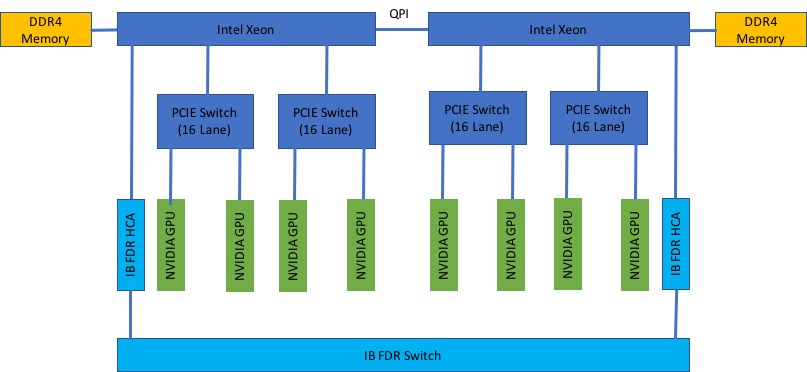}
	\caption{Architecture of a computing server. Each computing node has two Intel Xeon CPUs, each of which connects two PCIE switches and one IB HCA. Each PICE switch connects two Nvida Tesla P100 GPUs. The node is NUMA architecture. Each CPU has forty PCIE3 lanes, two PCIE switches takes thirty two lanes, and IB HCA uses remaining eight lanes.}
	\label{fig:archnode}
\end{figure}

\textbf{High Speed Interconnection} The data transfer system of our system is built with Mellanox FDR56 InfiniBand technology, which provides up to 56Gbit/s bandwidth and 4TB/s bi-direction throughput, and messages that control DL training through 10G Ethernet. As shown in Fig.~\ref{fig:archnode}, GPU cards connect CPU though PCIe 3. HCA cards connect nodes though IB switch. Since Nvidia P100 card supports GPUDirect remote direct memory access~(RDMA)that transfer data from P2P from GPU to GPUs directly, the memories of GPU and CPU of all nodes are broadly considered as one that connected by InfiniBand. 

\textbf{Three-level Data Cache} MiMatrix is designed a three-level data cache system. The first level is the memories of GPU and CPU connected with IB, in which the data is directly are feed into running DL training process. The second level is 2TB SSD of job server, in which data is the current training dataset that is loaded during running training tasks. The 2TB volume is enough for most of deep CNNs models training. And the third level is storage node which stores a variety of training datasets. Before a CNNs model is trained, the training data is copied to the second level data cache, SSD of the job server.

\subsection{Job Server and Software Implementation}
\label{sec:software}

\textbf{Job Server Framework} Parameter server framework is the most widely-used parallel DL software architecture~\cite{dean2012large,chen2015mxnet}. In it, a central node, referred to as parameter server, receives the model weights from all workers, and broadcasts the aggregating weights to all workers in each iteration, simultaneously. Two problems, stagger and bandwidth bottleneck, always lead to the poor performance.
 
According to hardware and topology of Manoa, we propose a novel job server software architecture, named MiMatrix. At the same time, we propose and implement a new AllReduce algorithm, \textbf{G}PU\textbf{D}irect \textbf{R}DMA-\textbf{A}ware \textbf{A}llReduce~(GDRAA), in which both computation and handshake message are $O(1)$, detailed in section~\ref{sec:algorithm}.

MiMatrix adopts message driven framework for the DL training. The job server acts on the central node and only receives and sends messages from and to computer servers and storage. The training process is executed by a protocol defined by users. Therefore, MiMatrix is flexible to any training approach with re-defining a new protocol. 

For updating model in each iteration, the computer sever is considered as both master and slaver. As illuminated in Fig.~\ref{fig:allreduce} and detail in Algorithm.~\ref{algorithm:DL}, each of computer severs not only sends part data of weights of this worker to other workers but also receives part data from other workers. After obtaining the data from all other computer severs, each server separately averages the data of the node and then broadcasts to other computer servers. Each of iteration only needs two synchronizations, which are the theoretical minimum number. 

\textbf{Software Implementation} We implement MiMatrix with C++11 in CPU part and CUDA 8.0~\cite{cudaprogram} with CUDNN 6.0 library~\cite{chetlur2014cudnn} on GPU part. The data transfer functions are directly written by \textit{ibverbs}~\cite{rdmaprogramming1}. The big advantage to implement on low level ibverbs api rather than MPI~\cite{mpiprogramming} or other InfiniBand's upper-layer protocols~(ULPs) such as over IB/SDP or RDS~\cite{rdmaprogramming,bedeir2010rdma,8291933} is to provide lower latency and allow for zero-copy transfer. 

In our implementation, there is a whole and continuous GPU memory registered by ibv\_reg\_mr, which is separated into two memories: Receiver Buffer~(RB) and Send Buffer~(SB), as shown in Fig.~\ref{fig:allreduce}. During running forward and backward part, the weights of the models are in SBs. The RB of each worker hold the parts of weights of other workers during weights aggregation. The data transfer directly calls ibv\_post\_send  among GPUs in different workers offloading CPU. 
\begin{figure}[htp!]
	\begin{minipage}[h]{1.0\linewidth}
		\centering
		\includegraphics[width=\linewidth, height=6cm]{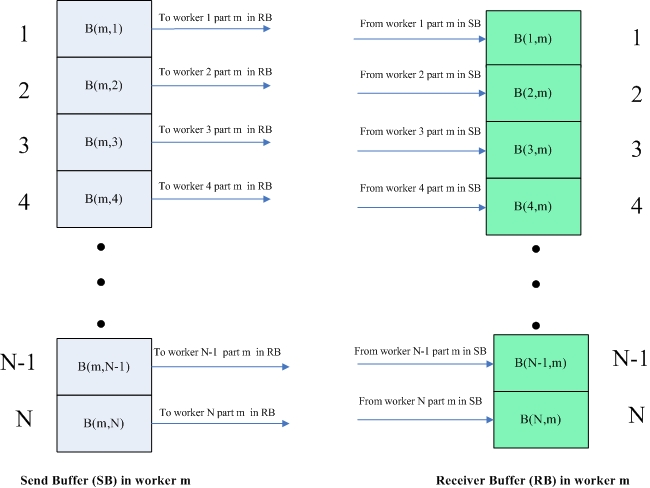}\\
		(a): Reduce Processing. 
	\end{minipage}
	\vspace{1.00mm}
	\begin{minipage}[h]{1.0\linewidth}
		\centering
		\includegraphics[width=\linewidth, height=6cm]{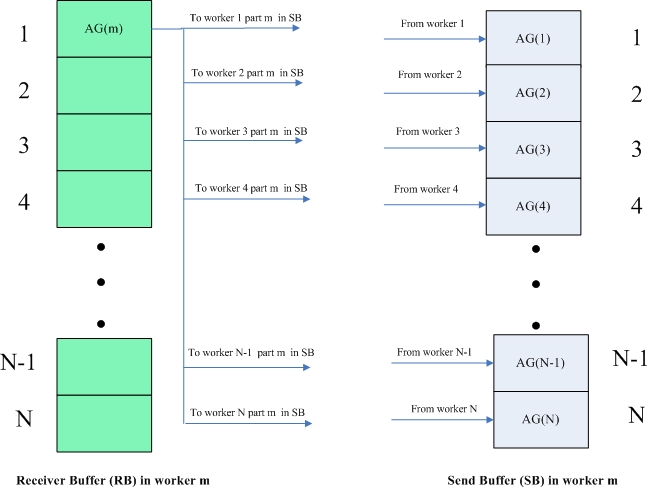}\\
		(b): Broadcast Processing. 
	\end{minipage}
	\hspace{0.00mm}
	\caption{Illumination of our proposed GDRAA AllReduce algorithm.  Example worker $m$ to explain our proposed algorithm. Image (a) shows reducing process. As a master, worker $m$ sends data in SB to other workers' RB, shown in left column. As a slaver, RB in worker $m$ receives data from other workers, shown in right column. Image (b) shows broadcast processing. After back propagation, the aggregating data of worker $m$ is in RB. As a master, the data is broadcasted to other workers, shown in left column. As a salver, worker $m$ receives data from other workers.}
	\label{fig:allreduce}
\end{figure}

\begin{algorithm}[!htp]
	\caption{GDRAA Algorithm.}
	\label{algorithm:DL}
	\begin{algorithmic}[1]
		\State Total N workers, each worker load different mini-batch data
		\State gradient of differential at worker $i$, $D(i)$
		\For {$Nu \leftarrow 0, EN$} \Comment{EN is maximum of iteration}
		\For {$i \leftarrow1, N $, each workers $W(i)$ parallel}
		\If {$Nu == 1$}
		\State Skip
		\Else
		\If {Worker $W(i)$ obtain $D(l), l\in[1,N]$ from all workers} 
		\State \Comment{1st synchronization}
		\EndIf
		\State \Comment{End of GDRW AllReduce}
		\State Update model with gradient of differential $D$ \Comment {Update the model}
		\EndIf
		\State Each worker $W(i)$ has model $M(Nu)$
		\State \Comment{Training task part}
		\State \Comment{Start of GDRW AllReduce part}
			\State Divide $D(i)$ is divided by $N$, and get $D(i,m), m\in[1,N]$ 
		\For {$j \leftarrow 1, N$}
		\State Send D(i, j) to worker $W(j)$\Comment{Shown in Fig.~\ref{fig:allreduce} (a)}
		\EndFor
		\If { Worker $W(i)$ obtain data $D(k,i), k\in[1,N]$ from all workers} 
		\State \Comment{2nd synchronization}
		\State Average $D(k,i), k\in[1,N ]$ and obtain $D(i)$
		\State send $D(i)$ to all workers \Comment{Shown in Fig.~\ref{fig:allreduce} (b)}
		\EndIf
		\EndFor	
		\EndFor
	\end{algorithmic}
\end{algorithm}

\section{GDRAA Algorithm}
\label{sec:algorithm}
In this section, we detail our proposed AllReduce algorithm, \textbf{G}PU\textbf{D}irect \textbf{R}DMA-\textbf{A}ware \textbf{A}llReduce(GDRAA), in Algorithm~\ref{algorithm:DL}. Basic idea is illuminated in Fig.~\ref{fig:allreduce}. Then, we prove that GDRAA is $O(1)$ in both computational complex and handshake message for the number of worker $N$. 

In real-word applications, the operator of data copy is compose of data transfer and latency. To be sure that the time of the data transfer dominates the whole time of data copy, our proposed algorithm is based on two assumptions:
\begin{enumerate}
	\item Compared to time of data transfer, the latency is tiny.
	\item The number of worker is small to  guarantee the summary of latency operators is small as well.
\end{enumerate}
The latency of data copy of Mellanox FDR56 IB is 0.7 $\mu$sec, and MiMatrix is designed for maximum 32 workers. Obviously, that our design satisfies the above two assumptions. 

GDRAA is a three-step algorithm: 1): Reduce, each workers obtains weights from all workers; 2): Aggregation, each worker averages these weights that it achieves; and 3): Broadcast, each worker sends the averaging data to all workers. Each worker allocate a continuous memory dividing two segmentations, send buffer (SB) and receiver buffer (RB), respectively. Both are divided into $N$ if the system has $N$ workers. As shown in Fig.\ref{fig:allreduce}, in Reduce step, the SB has $N$ part data, $B(i, j)$, $i$ is the worker number, $i\in[1,N]$ and $j$ is block number, $j\in[1,N]$. $B(i, j)$ is sent to block number $i$  at worker $j$. Take worker $m$ as a example, SB has N block data, $B(m, j)$,  $j\in[1,N]$ and this data is sent to worker $i$, block $m$. Meanwhile, $RB$ will receive $N$ data from $N$ workers. As shown in Fig.~\ref{fig:allreduce} (a) right column, RB will receives N data $B(j,m)$, $j$ means worker number. . For each worker, once obtaining data from all workers, the weights are averaged. Then it is the broadcast part, as shown in Fig.~\ref{fig:allreduce} (b), the worker sends the averaged result to all others. Meanwhile, the worker achieves the averaged data from other workers. We also take worker $m$ as a example, $AG(m)$ is the average data in worker $m$ and will broadcast to other workers at block $m$. Meanwhile, its SB receives all $AG(i)$, $i\in[1,N]$ from other workers. 

GDRAA operates reduce, aggregation and broadcast asynchronously, and has two synchronizations to wait data of each iteration. The synchronizations are designed based on  DL training instead of instead of AllReduce that most of existing distributed DL approaches have used. The synchronizations operator of GDRAA is as latest as the DL needs.  Consequently, the MiMatrix has greatly relieved the stagger problem.

The rest of this section, we proof that GDRAA has $O(1)$ in both computation and handshake message.
\begin{lem}
\label{lem:1}
	For a distributed system with $N$ workers, given the fixed size of memory of each worker, $L$. The handshake message is $O(1)$ of $N$. 
	\end{lem}

\begin{proof}
\label{proof:1}
	For any worker, $m$, the data is divided into N. In the first step, considered worker $m$, the worker will send $L/N$ to other workers. Therefore, the size of data sent is: 
	\begin{equation}
	\label{eq:1}
	\frac{L}{N} \times (N-1) = L \times (1- \frac{1}{N} )\\
  	\simeq L 
	\end{equation}
	At the same time, the worker $m$ is receiving $\frac{L}{N}$ from total $N-1$. Therefore, the size of data received is
	\begin{equation}
	\label{eq:2}
	\frac{L}{N} \times (N-1) = L \times (1- \frac{1}{N} )\\
	\simeq L
	\end{equation}
	The Equation~\ref{eq:1} and Equation~\ref{eq:2} have shown that in first step before model data averaging, the data that worker $m$ sends and receives is both about $L$ data, which doesn't depend on the size of $N$.
	
	For the step 2, the proof the similar to the process of step 1.  
\end{proof}

\begin{lem}
	\label{lem:2}
	For a distributed system with $N$ workers, given any worker, the the computation complex of aggregation is  $O(1)$ of $N$. 
\end{lem}

\begin{proof}
\label{proof:2}
	In the aggregation of the distributed DL training, the operator is averaging. For any worker, $m$, the computational complex $Op$equals to
\begin{equation}
	\label{eq:3}
	Op = \frac{L}{N} \times (N-1) + \frac{L}{N}
\end{equation}
In Equation~\ref{eq:3}, the first part is the complex of add operator, and the second part is the complex of the multiply operator. We can rewrite it as:
 \begin{equation}
 \label{eq:4}
  Op= L
\end{equation}
For number of worker, $N$, the $L$ is a fixed. As a result, the computational complex is $O(1)$ and independence to $N$. 
\end{proof}

\section{Experimental Results and Analysis}
\label{sec:results}
\subsection{Configuration}
\label{sec:conf}
We took two Intel Xeon E5-2650V4 CPUs with 512GB memory in job sever and two Intel Xeon E5-2620V4 CPUs with 256GB memory and eight Nvidia Tesla P100 GPU with 16GB memory in the computing server. In computer servers, most of computational tasks running on GPUs. Low-end CPUs and low volume memory is enough. This design greatly reduces the cost of the system and increase PCR.

Manoa has sixteen computer servers and equipped in two 48U racks. Both job server and computing sever are 4U node. Operating system is CentOS 7. Version of the kernel is $3.10$. Version of Nvidia driver is $384.11$. The version of GCC is $4.8$.

In our training procedure, some parameters are set as following: the initial learning rate is $0.1$, momentum is $0.9$, weight decay is $0.001$, learning rate change policy is "ploy" and the gamma is $1$. Recently, some researchers has represented some tricks to quickly train extremely large minibatch~\cite{goyal2017accurate,you2017imagenet,akiba2017extremely}. Obviously, some of them can greatly speedup our training tasks. However, in this paper, we aim at benching the job server performance and present reproductive works. We don't try these tricks.  

\subsection{Results and Analysis}
\label{sec:exper}
As mentioned in section~\ref{sec:system}, the system performance was evaluated by the training time reaching an desirable  accuracy. In our paper, we conducted our experiments on ImageNet-1K dataset~\cite{russakovsky2015imagenet}, and we took ResNet with batch normalization~\cite{batch} our benchmark~\cite{he2016deep}.  We did two types of tests: 1) PRC measurement. The training task has 32 workers, each of which has 4GPUs. In our experiment, each GPU has 40 images. and 2) Scaling efficient. We measured the scaling efficient using Resnet50 on ImageNet-1K with 1,2,4,8,16 and 32 workers. The time obtained is that Resnet50 reaching 65\% on validation data, in which each GPU has 64 images and each worker has 4 GPUs\footnote {\label{footnote_time}All training times listed in this paper include saving models time of each epoch.}.

\begin{figure}[!t]
	\centering
	\includegraphics[width=0.9 \textwidth]{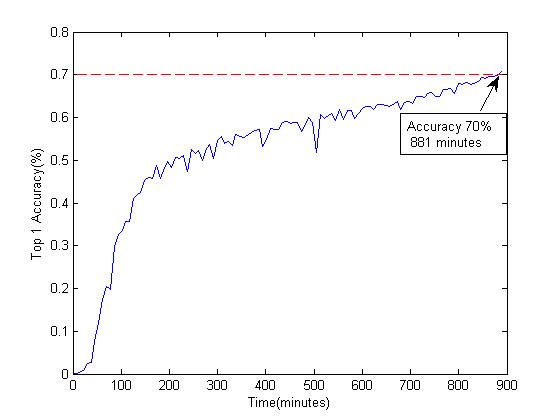}
	\caption{ResNet101 accuracy on validation data of ImageNet 1K. The mini-batch is 1,280. The time reaching 70\% accuracy is 888 minutes.}
	\label{fig:resnet101}
\end{figure}
As shown in Figure.~\ref{fig:resnet101}, Training time of ResNet101 models reaching 70\% accuracy is 888 minutes. The total price of Manoa is 900K dollar. The PRC of Manoa is 0.00000125. 
\begin{table}[!t]
	\label{table:scal}
	\caption{Scaling Efficiency of Manoa.}
	\centering
	\begin{tabular}{|c|c|c|}
		\hline
		Number of Worker& Time (minutes) &Scalability\\
		\hline
		1&4841&1\\
		\hline
		2&3039&1.59\\
		\hline
		4&1644&2.84\\
		\hline
		8&850&5.70\\
		\hline
		16&430&11.26\\
		\hline
		32&333&14.56\\
		\hline
	\end{tabular}
\end{table}
\section{Conclusion}
\label{sec:conclusion}
In this paper, we propose a co-design distributed heterogeneous cluster for speeding up DL training. In it, we successfully design and build a high-density 128-GPU cluster, named Manoa, and propose and develop a job server parallel framework for DL training, named MiMatrix, which effectively and efficiently utilizes the whole system. Our system achieves high ratio between speed converge and price. Compared to the parameter sever framework, the job server framework successfully solve the bandwidth bottleneck and stagger problems. Meanwhile, we also propose a novel AllReduce algorithm, GDRAA, which has $O(1)$ in both computation and handshake message. We conducted our experiments on ImageNet-1K dataset, and the performance has demonstrated state-of-the-art.

In the future, we are planning to investigate better hyper parameters such as adaptive learning rate policy and monument to improve our system performance. Also, we will obtain much results on larger-scale datasets, such as ImageNet22K and/or Places-365~\cite{place2} datasets. 

\section*{Acknowledgements}
Some of the technology described in this paper is patent pending. Many thanks to Super Micro Computer, Inc. for their help with building Manoa.

\section*{References}

\bibliography{gpuclusterbib,deeplearningbib}

\begin{thebibliography}{10}
\expandafter\ifx\csname url\endcsname\relax
  \def\url#1{\texttt{#1}}\fi
\expandafter\ifx\csname urlprefix\endcsname\relax\def\urlprefix{URL }\fi
\expandafter\ifx\csname href\endcsname\relax
  \def\href#1#2{#2} \def\path#1{#1}\fi

\bibitem{lecun2015deep}
Y.~LeCun, Y.~Bengio, G.~Hinton, Deep learning, Nature 521~(7553) (2015)
  436--444.

\bibitem{schmidhuber2015deep}
J.~Schmidhuber, Deep learning in neural networks: An overview, Neural Networks
  61 (2015) 85--117.

\bibitem{liu2017survey}
W.~Liu, Z.~Wang, X.~Liu, N.~Zeng, Y.~Liu, F.~E. Alsaadi, A survey of deep
  neural network architectures and their applications, Neurocomputing 234
  (2017) 11--26.

\bibitem{krizhevsky2012imagenet}
A.~Krizhevsky, I.~Sutskever, G.~E. Hinton, Imagenet classification with deep
  convolutional neural networks, in: Advances in neural information processing
  systems, 2012, pp. 1097--1105.

\bibitem{he2016deep}
K.~He, X.~Zhang, S.~Ren, J.~Sun, Deep residual learning for image recognition,
  in: Proceedings of the IEEE conference on computer vision and pattern
  recognition, 2016, pp. 770--778.

\bibitem{bengio2009learning}
Y.~Bengio, Learning deep architectures for {AI}, Foundations and
  trends{\textregistered} in Machine Learning 2~(1) (2009) 1--127.

\bibitem{hinton2012deep}
G.~Hinton, L.~Deng, D.~Yu, G.~E. Dahl, A.-r. Mohamed, N.~Jaitly, A.~Senior,
  V.~Vanhoucke, P.~Nguyen, T.~N. Sainath, et~al., Deep neural networks for
  acoustic modeling in speech recognition: The shared views of four research
  groups, Signal Processing Magazine, IEEE 29~(6) (2012) 82--97.

\bibitem{amodei2016deep}
D.~Amodei, S.~Ananthanarayanan, R.~Anubhai, J.~Bai, E.~Battenberg, C.~Case,
  J.~Casper, B.~Catanzaro, Q.~Cheng, G.~Chen, et~al., Deep speech 2: End-to-end
  speech recognition in english and mandarin, in: International Conference on
  Machine Learning, 2016, pp. 173--182.

\bibitem{silver2016mastering}
D.~Silver, A.~Huang, C.~J. Maddison, A.~Guez, L.~Sifre, G.~Van Den~Driessche,
  J.~Schrittwieser, I.~Antonoglou, V.~Panneershelvam, M.~Lanctot, et~al.,
  Mastering the game of go with deep neural networks and tree search, nature
  529~(7587) (2016) 484--489.

\bibitem{silver2017mastering}
D.~Silver, J.~Schrittwieser, K.~Simonyan, I.~Antonoglou, A.~Huang, A.~Guez,
  T.~Hubert, L.~Baker, M.~Lai, A.~Bolton, et~al., Mastering the game of go
  without human knowledge, Nature 550~(7676) (2017) 354.

\bibitem{russakovsky2015imagenet}
O.~Russakovsky, J.~Deng, H.~Su, J.~Krause, S.~Satheesh, S.~Ma, Z.~Huang,
  A.~Karpathy, A.~Khosla, M.~Bernstein, et~al., Imagenet large scale visual
  recognition challenge, International Journal of Computer Vision 115~(3)
  (2015) 211--252.

\bibitem{cocodatset}
T.-Y. Lin, M.~Maire, S.~Belongie, J.~Hays, P.~Perona, D.~Ramanan,
  P.~Doll{\'a}r, C.~L. Zitnick, Microsoft coco: Common objects in context, in:
  ECCV, 2014, pp. 740--755.

\bibitem{simonyan2014very}
K.~Simonyan, A.~Zisserman, Very deep convolutional networks for large-scale
  image recognition, in: Proc. International Conference on Learning
  Representations, 2015.

\bibitem{huang2017densely}
G.~Huang, Z.~Liu, L.~van~der Maaten, K.~Q. Weinberger, Densely connected
  convolutional networks, in: Proceedings of the IEEE Conference on Computer
  Vision and Pattern Recognition, 2017.

\bibitem{ruder2016overview}
S.~Ruder, An overview of gradient descent optimization algorithms, arXiv
  preprint arXiv:1609.04747.

\bibitem{zhang2018theory}
C.~Zhang, Q.~Liao, A.~Rakhlin, B.~Miranda, N.~Golowich, T.~Poggio, Theory of
  deep learning iib: Optimization properties of {SGD}, arXiv preprint
  arXiv:1801.02254.

\bibitem{daniely2017sgd}
A.~Daniely, {SGD} learns the conjugate kernel class of the network, arXiv
  preprint arXiv:1702.08503.

\bibitem{revistingSGD}
J.~Chen, R.~Monga, S.~Bengio, R.~Jozefowicz, Revisiting distributed synchronous
  {SGD}, in: International Conference on Learning Representations Workshop
  Track, 2016.

\bibitem{distilling}
G.~Hinton, O.~Vinyals, J.~Dean, Distilling the knowledge in a neural network,
  in: NIPS Deep Learning and Representation Learning Workshop, 2015.

\bibitem{han2015learning}
S.~Han, J.~Pool, J.~Tran, W.~Dally, Learning both weights and connections for
  efficient neural network, in: Advances in neural information processing
  systems, 2015, pp. 1135--1143.

\bibitem{Laine2017}
P.~Molchanov, S.~Tyree, T.~Karras, T.~Aila, J.~Kautz, Pruning convolutional
  neural networks for resource efficient inference, 2017.

\bibitem{szegedy2017inception}
C.~Szegedy, S.~Ioffe, V.~Vanhoucke, A.~A. Alemi, Inception-v4, inception-resnet
  and the impact of residual connections on learning., in: AAAI, Vol.~4, 2017,
  p.~12.

\bibitem{masters2018revisiting}
D.~Masters, C.~Luschi, Revisiting small batch training for deep neural
  networks, arXiv preprint arXiv:1804.07612.

\bibitem{cho2017powerai}
M.~Cho, U.~Finkler, S.~Kumar, D.~Kung, V.~Saxena, D.~Sreedhar, Power{AI} {DDL},
  arXiv preprint arXiv:1708.02188.

\bibitem{zhang2016parallel}
J.~Zhang, C.~De~Sa, I.~Mitliagkas, C.~R{\'e}, Parallel sgd: When does averaging
  help?, arXiv preprint arXiv:1606.07365.

\bibitem{rdmaprogramming1}
{Mellanox Technology RDMA Aware Programming user manual},
  \url{http://www.mellanox.com/related-docs/prod_software/RDMA_Aware_Programming_user_manual.pdf},
  last accessed March 3, 2018.

\bibitem{coates2013deep}
A.~Coates, B.~Huval, T.~Wang, D.~Wu, B.~Catanzaro, N.~Andrew, Deep learning
  with {COTS HPC} systems, in: International Conference on Machine Learning,
  2013, pp. 1337--1345.

\bibitem{chilimbi2014project}
T.~M. Chilimbi, Y.~Suzue, J.~Apacible, K.~Kalyanaraman, Project adam: Building
  an efficient and scalable deep learning training system., in: OSDI, Vol.~14,
  2014, pp. 571--582.

\bibitem{kurth2017deep}
T.~Kurth, J.~Zhang, N.~Satish, E.~Racah, I.~Mitliagkas, M.~M.~A. Patwary,
  T.~Malas, N.~Sundaram, W.~Bhimji, M.~Smorkalov, et~al., Deep learning at
  {15PF}: supervised and semi-supervised classification for scientific data,
  in: Proceedings of the International Conference for High Performance
  Computing, Networking, Storage and Analysis, ACM, 2017, p.~7.

\bibitem{bhattacharjee2017distributed}
B.~Bhattacharjee, M.~Hill, H.~Wu, P.~Chandakkar, J.~Smith, M.~Wegman,
  Distributed learning of deep feature embeddings for visual recognition tasks,
  IBM Journal of Research and Development 61~(4) (2017) 4--1.

\bibitem{dean2012large}
J.~Dean, G.~Corrado, R.~Monga, K.~Chen, M.~Devin, M.~Mao, A.~Senior, P.~Tucker,
  K.~Yang, Q.~V. Le, et~al., Large scale distributed deep networks, in:
  Advances in neural information processing systems, 2012, pp. 1223--1231.

\bibitem{chen2015mxnet}
T.~Chen, M.~Li, Y.~Li, M.~Lin, N.~Wang, M.~Wang, T.~Xiao, B.~Xu, C.~Zhang,
  Z.~Zhang, Mxnet: A flexible and efficient machine learning library for
  heterogeneous distributed systems, in: In NIPS Workshop on Machine Learning
  Systems (LearningSys), 2016.

\bibitem{caffe2}
{Caffe2}, \url{https://caffe2.ai/}, last accessed March 23, 2018.

\bibitem{Paddlepaddle}
Paddlepaddle, \url{https://github.com/PaddlePaddle/Paddle}, last access March
  25, 2018.

\bibitem{MPItheano}
H.~Ma, F.~Mao, G.~W. Taylor, Theano-{MPI}: A theano-based distributed training
  framework, in: F.~Desprez, P.-F. Dutot, C.~Kaklamanis, L.~Marchal,
  K.~Molitorisz, L.~Ricci, V.~Scarano, M.~A. Vega-Rodr{\'i}guez, A.~L.
  Varbanescu, S.~Hunold, S.~L. Scott, S.~Lankes, J.~Weidendorfer (Eds.),
  Euro-Par 2016: Parallel Processing Workshops, Springer International
  Publishing, 2017, pp. 800--813.

\bibitem{tensorflow}
M.~Abadi, P.~Barham, J.~Chen, Z.~Chen, A.~Davis, J.~Dean, M.~Devin,
  S.~Ghemawat, G.~Irving, M.~Isard, M.~Kudlur, J.~Levenberg, R.~Monga,
  S.~Moore, D.~G. Murray, B.~Steiner, P.~Tucker, V.~Vasudevan, P.~Warden,
  M.~Wicke, Y.~Yu, X.~Zheng, Tensorflow: A system for large-scale machine
  learning, in: Proceedings of the 12th USENIX Conference on Operating Systems
  Design and Implementation, OSDI'16, USENIX Association, Berkeley, CA, USA,
  2016, pp. 265--283.

\bibitem{CNTK}
{CNTK}, \url{https://github.com/Microsoft/CNTK}, last access March 16, 2018.

\bibitem{xing2015petuum}
E.~P. Xing, Q.~Ho, W.~Dai, J.~K. Kim, J.~Wei, S.~Lee, X.~Zheng, P.~Xie,
  A.~Kumar, Y.~Yu, Petuum: A new platform for distributed machine learning on
  big data, IEEE Transactions on Big Data 1~(2) (2015) 49--67.

\bibitem{nccl}
Nvidia collective communications library ({NCCL}),
  \url{https://developer.nvidia.com/nccl}, last access March 26, 2018.

\bibitem{chetlur2014cudnn}
S.~Chetlur, C.~Woolley, P.~Vandermersch, J.~Cohen, J.~Tran, B.~Catanzaro,
  E.~Shelhamer, cudnn: Efficient primitives for deep learning, arXiv preprint
  arXiv:1410.0759.

\bibitem{chen2014big}
X.-W. Chen, X.~Lin, Big data deep learning: challenges and perspectives, IEEE
  access 2 (2014) 514--525.

\bibitem{cudaprogram}
Nvidia, {CUDA C} programming guide,
  \url{http://docs.nvidia.com/cuda/pdf/CUDA_C_Programming_Guide.pdf}, last
  accessed December, 2017.

\bibitem{mpiprogramming}
M.~P.~I. Forum, {MPI}: A message-passing interface standard version 3.0,
  \url{http://mpi-forum.org/docs/mpi-3.0/mpi30-report.pdf}, last accessed March
  23, 2018.

\bibitem{rdmaprogramming}
M.~Technology, {RDMA} aware programming user manual,
  \url{http://www.mellanox.com/related-docs/prod_software/RDMA_Aware_Programming_user_manual.pdf},
  last accessed March 8, 2018.

\bibitem{bedeir2010rdma}
T.~Bedeir, Rdma read and write with ib verbs, Tech. rep., Technical report, HPC
  Advisory Council, 2010. URL: http://www. hpcadvisorycouncil.
  com/pdf/rdma-read-and-write-with-ib-verbs. pdf (2010).

\bibitem{8291933}
Y.~Ren, X.~Wu, L.~Zhang, Y.~Wang, W.~Zhang, Z.~Wang, M.~Hack, S.~Jiang, irdma:
  Efficient use of {RDMA} in distributed deep learning systems, in: 2017 IEEE
  19th International Conference on High Performance Computing and
  Communications; IEEE 15th International Conference on Smart City; IEEE 3rd
  International Conference on Data Science and Systems (HPCC/SmartCity/DSS),
  2017, pp. 231--238.

\bibitem{goyal2017accurate}
P.~Goyal, P.~Doll{\'a}r, R.~Girshick, P.~Noordhuis, L.~Wesolowski, A.~Kyrola,
  A.~Tulloch, Y.~Jia, K.~He, Accurate, large minibatch sgd: training imagenet
  in 1 hour, arXiv preprint arXiv:1706.02677.

\bibitem{you2017imagenet}
Y.~You, Z.~Zhang, C.~Hsieh, J.~Demmel, K.~Keutzer, Imagenet training in
  minutes, CoRR, abs/1709.05011.

\bibitem{akiba2017extremely}
T.~Akiba, S.~Suzuki, K.~Fukuda, Extremely large minibatch sgd: Training
  resnet-50 on imagenet in 15 minutes, arXiv preprint arXiv:1711.04325.

\bibitem{batch}
S.~Ioffe, C.~Szegedy, Batch normalization: Accelerating deep network training
  by reducing internal covariate shift, in: Proceedings of the 32Nd
  International Conference on International Conference on Machine Learning -
  Volume 37, ICML'15, 2015, pp. 448--456.

\bibitem{place2}
B.~Zhou, A.~Lapedriza, A.~Khosla, A.~Oliva, A.~Torralba, Places: A 10 million
  image database for scene recognition, IEEE Transactions on Pattern Analysis
  and Machine Intelligence 40~(6) (2018) 1452--1464.

\end{thebibliography}

\end{document}